\documentclass{article}

\PassOptionsToPackage{numbers, compress}{natbib}
\usepackage{graphicx}
\usepackage[main,final]{neurips_2026}
\usepackage[utf8]{inputenc} 
\usepackage[T1]{fontenc}    
\usepackage{hyperref} 
\usepackage{wrapfig}
\usepackage{url}            
\usepackage{booktabs}       
\usepackage{amsfonts}       
\usepackage{nicefrac}       
\usepackage{microtype}      
\usepackage{xcolor}         
\usepackage{amsthm} 
\usepackage{amsmath}   
\usepackage{amsfonts} 
\usepackage{amssymb}   
\usepackage{mathtools}

\newtheorem{theorem}{Theorem}[section]     
\newtheorem{lemma}[theorem]{Lemma}

\theoremstyle{definition}                   
\newtheorem{definition}{Definition}[section]

\theoremstyle{remark}

\usepackage{thmtools}
\usepackage{thm-restate}

\usepackage{enumitem}

\usepackage[font=small]{caption}

\usepackage[normalem]{ulem}
\newcommand\tsout{\bgroup\markoverwith{\textcolor{magenta}{\rule[0.5ex]{2pt}{0.8pt}}}\ULon}

\usepackage[capitalize,noabbrev]{cleveref}
\definecolor{refcolor}{HTML}{000099}
\hypersetup{ %
pdftitle={},
pdfauthor={},
pdfsubject={},
pdfkeywords={},
pdfborder=0 0 0,
pdfpagemode=UseNone,
colorlinks=true,
linkcolor=refcolor,
citecolor=refcolor,
filecolor=refcolor,
urlcolor=refcolor,    
pdfview=FitH}

\usepackage[most]{tcolorbox}

\newtcbtheorem[
    auto counter,
    number within=section
]{boxeddefinition}{Definition}{
    enhanced,
    breakable,
    colback=white,
    colframe=black,
    boxrule=0.5pt,
    arc=0pt,
    fonttitle=\bfseries,
    coltitle=black,
    attach title to upper,
    separator sign={.~},
}{def}

\usepackage{titlesec}
\titlespacing\section{0pt}{0pt plus 0pt minus 1pt}{0pt plus 0pt minus 1pt}
\titlespacing\subsection{0pt}{0pt plus 0pt minus 1pt}{0pt plus 0pt minus 1pt}
\titlespacing\subsubsection{0pt}{0pt plus 0pt minus 1pt}{0pt plus 0pt minus 1pt}

\usepackage{titletoc}

\newcommand\DoToC{%
  \startcontents
  \printcontents{}{1}{\textbf{Table of Contents}\vskip3pt\hrule\vskip5pt}
  \vskip3pt\hrule\vskip5pt
}

\title{ControlJEPA: Principled Trajectory Regularization via Lyapunov Tube Loss}

\author{Thai-Khanh Nguyen$^{\ast}$ \\
Faculty of Information Technology 1 \\
Posts \& Telecommunications Institute of Technology \\
\texttt{khanhnt1@ptit.edu.vn}
\And
Nhi Y. Pham$^{\ast}$ \\
Center for AI Research \\
VinUniversity \\
\texttt{nhi.py2@vinuni.edu.vn}
\AND
Uyen N.B. Vo$^{\ast}$ \\
Faculty of Artificial Intelligence \\
Posts \& Telecommunications Institute of Technology\\
\texttt{uyenvnb@ptit.edu.vn}
\And
Thieu N. Vo$^{\dagger}$ \\
Department of  Computer Science \\
University of Bath \\
\texttt{ntv22@bath.ac.uk}
\AND
Cuong Pham$^{\dagger}$ \\
Faculty of Information Technology 1 \\
Posts \& Telecommunications Institute of Technology\\
\texttt{cuongpv@ptit.edu.vn}
\And
Tan M. Nguyen$^{\dagger}$ \\
Department of Mathematics \\
National University of Singapore \\
\texttt{tanmn@nus.edu.sg}
}

\newcommand\blfootnote[1]{%
  \begingroup
  \renewcommand\thefootnote{}\footnote{#1}%
  \addtocounter{footnote}{-1}%
  \endgroup
}

\begin{document}

\maketitle


\begin{abstract}
Next-token prediction leaves hidden-state trajectories largely unconstrained, while existing geometric regularizers primarily control the direction rather than the magnitude of trajectory deviations. We introduce \emph{ControlJEPA}, a trajectory regularization framework based on a discrete-time Lyapunov contraction condition from practical Input-to-State Stability theory. Its Lyapunov Tube Loss penalizes violations of this condition on a normalized, scale-invariant measure of transversal deviation. When the contraction condition holds at every step, the hidden-state trajectory remains within an explicit tube around the reference secant. In the partial-ISS regime, we bound mean transversal energy, temporal variance, and average tube width using the Lyapunov loss and contraction parameters, with horizon-independent variance bounds under uniformly bounded violations. ControlJEPA requires no architectural modifications. Across six benchmarks and six model families, it consistently improves accuracy and data efficiency over standard fine-tuning and prior trajectory regularizers. The resulting smoother transversal profile also persists empirically during inference, when the regularizer is inactive. The code is publicly available at: \url{https://github.com/t-khanusa/ControlJEPA}.\blfootnote{$^{\ast}$Co-first authors. $^{\dagger}$Co-last authors. Correspondence to: khanhnt1@ptit.edu.vn \& tanmn@nus.edu.sg}
\end{abstract}

\section{Introduction}
A fundamental limitation of next-token prediction (NTP) is that it optimizes predictions at each position independently, without imposing global structure on the resulting hidden-state trajectories~\citep{schmidt2019generalization, huang2026semantic}. Recent trajectory regularization methods seek to address this limitation by penalizing angular deviations between trajectory segments defined by randomly sampled triplets of hidden states, thereby improving data efficiency without modifying the model architecture~\citep{huang2026semantic, wang2026temporal}. However, these angle-based approaches constrain only the direction of displacement, rather than the magnitude; as a result, trajectories may still drift arbitrarily far from the semantic reference path without incurring any penalty.

This limitation reflects a deeper issue: the lack of a principled stability framework for hidden-state dynamics. Existing regularizers provide neither monotonic control over transversal deviation nor explicit bounds on trajectory divergence as a function of hyperparameters. As a result, notions such as tube width and regularization strength remain ill-defined.

{\bf Contribution.} We introduce the \textit{Controlled Joint Embedding Predictive Architecture (ControlJEPA)}, a principled trajectory regularization framework grounded in discrete-time Input-to-State Stability (ISS) theory~\citep{JIANG19992403}. Our central contribution is the \textit{Lyapunov Tube Loss} $\mathcal{L}_{\mathrm{Lyap}}$, which penalizes violations of an ISS contraction condition on a dimensionless measure of transversal deviation and is zero whenever the condition is already satisfied. ControlJEPA augments standard NTP training with $\mathcal{L}_{\mathrm{Lyap}}$:
\begin{equation}
    \mathcal{L}_{\mathrm{ControlJEPA}}(\theta) 
        = \mathcal{L}_{\mathrm{NTP}}(\theta) 
        + \lambda\,\mathcal{L}_{\mathrm{Lyap}}(\theta),
\end{equation}
requiring no architectural modifications. When the contraction condition holds, hidden states remain within a tube around a reference direction with radius $r^*=\sqrt{\tau/(1-\gamma)}||v_{\mathrm{geo}}||_2$ (Definition~\ref{def:lyapunov_tube}). Enforcing it at every step, however, can excessively compress the transversal component and limit representational capacity (Section~\ref{sec:capacity}). We therefore operate in a \emph{partial-ISS} regime, where the constraint is enforced softly and the achieved geometry reflects a balance between NTP and Lyapunov regularization. The effective tube width is mainly controlled by $\lambda$ and $\gamma$, while $\tau$ sets the theoretical ceiling $r^*$. Normalization by the reference chord further makes the loss invariant to global rescaling. Our contributions are three-fold:

\begin{enumerate}[leftmargin=25pt]
    \item We propose a stability-based framework for hidden-state trajectory regularization using discrete-time ISS theory, yielding explicit and interpretable bounds on trajectory deviation.
    \item We introduce the Lyapunov Tube Loss, a practical and architecture-agnostic regularization objective that enforces stability while preserving representational flexibility.
    \item We bound mean transversal energy, temporal variance, and average tube width in the partial-ISS regime, with horizon-independent variance control under uniformly bounded maximal violations. A spectral capacity bound links vanishing transversal deviation to limited logit expressivity under a bounded unembedding norm, motivating soft regularization.
\end{enumerate}

We empirically demonstrate that ControlJEPA consistently outperforms standard fine-tuning and prior trajectory regularizers across six benchmarks, including NL-RX-SYNTH, NL-RX-TURK, GSM8K, Spider, NQ-Open, and HellaSwag, and six model families. On GSM8K and NQ-Open, it improves accuracy over standard fine-tuning by approximately 10 percentage points on each benchmark, and over Semantic Tube Prediction by approximately 6 and 4 percentage points, respectively. ControlJEPA also improves data efficiency, exhibits a stable hyperparameter region, and empirically retains smoother transversal profiles during inference, when the regularizer is inactive.

{\bf Organization.} We structure this paper as follows: In Section~\ref{sec:preliminaries}, we formalize hidden-state trajectories and introduce the discrete-time ISS framework. In Section~\ref{sec:method}, we derive the Lyapunov Tube Loss and the ControlJEPA objective. In Section~\ref{sec:theoretical_analysis}, we establish theoretical guarantees on capacity preservation and transversal profile stability. Sections~\ref{sec:experiments} and ~\ref{sec:analysis} present empirical validation and deeper analysis across six benchmarks and six model families. We discuss related work in Section ~\ref{sec:related_work} and conclude with future directions. Proofs and implementation details are provided in the Appendix.

\section{Preliminaries}
\label{sec:preliminaries}
\subsection{Hidden-State Trajectories in Language Models}

Let $h_t \in \mathbb{R}^D$ denote the hidden-state representation at token position $t$, defined as the final transformer layer output prior to the output projection (unembedding) into vocabulary logits. A transformer trained with next-token prediction (NTP) induces a sequence of hidden states $h_t = f_\theta(x_{\leq t})$ that evolves according to a discrete-time dynamical system:
\begin{equation}
    h_{t+1} = F_\theta(h_t, x_{t+1}), \qquad t \in \{0, \ldots, T-1\},
\end{equation}
where $\theta$ denotes the model parameters and $x_{t+1}$ is the input at step $t+1$ (the ground-truth token under teacher forcing). The initial state $h_0$ is determined by the input embedding. In the terminology of dynamical systems, this defines a discrete-time \emph{map}~\citep{wiggins2003introduction}.

Training is performed via the NTP objective:
\begin{equation}
    \mathcal{L}_{\mathrm{NTP}}(\theta) 
    = -\frac{1}{T} \sum_{t=1}^{T} \log p_\theta(x_t \mid x_{<t}),
\end{equation}
which optimizes predictions independently at each position. As a consequence, it imposes no explicit constraint on the geometry of the resulting hidden-state trajectory $(h_0, \ldots, h_T)$, which is the primary object of interest in this work.

\subsection{Existing Trajectory Regularization}

The \emph{geodesic hypothesis} suggests that well-structured hidden-state trajectories should follow approximately straight paths in representation space, as straighter trajectories enable more predictable dynamics and reduce first-order prediction error~\citep{henaff2019perceptual}. This behavior is observed in both biological and artificial systems, and recent evidence indicates that large language models implicitly learn to straighten representation trajectories, suggesting that linearization is a fundamental inductive bias in sequence modeling~\citep{hosseini2023large}. Recent methods operationalize this idea using angle-based penalties.

{\bf Semantic Tube Prediction (STP).} \citet{huang2026semantic} penalizes angular deviation between displacement vectors constructed from randomly sampled triplets:
\begin{equation}
    \mathcal{L}_{\mathrm{STP}} 
    = \mathbb{E}_{s < r < t} 
    \bigl[1 - \cos(h_t - h_r,\; h_r - h_s)\bigr].
\end{equation}

{\bf Temporal Straightening (TS).} \citet{wang2026temporal} enforces local smoothness by penalizing curvature between consecutive latent velocities:
\begin{equation}
    \mathcal{L}_{\mathrm{TS}} 
    = \mathbb{E}_t 
    \bigl[1 - \cos(z_{t+1} - z_t,\; z_{t+2} - z_{t+1})\bigr].
\end{equation}

While effective in practice, these approaches share key limitations. First, they constrain only the \emph{direction} of motion and do not control the \emph{magnitude} of deviation, allowing trajectories to drift arbitrarily far from a reference path. Second, they operate locally and do not enforce global bounds on trajectory deviation. Finally, they provide no explicit relationship between hyperparameters and geometric quantities such as tube width, limiting interpretability.

\subsection{Discrete-Time Input-to-State Stability}
\label{sec:prelim:maps}
To overcome the aforementioned limitations, we adopt a discrete-time Lyapunov contraction condition inspired by Input-to-State Stability (ISS) theory \citep{JIANG19992403}. Unlike angle-based penalties, this condition controls the magnitude of deviation and yields explicit bounds on it. The discrete-time formulation is natural for transformers and JEPA-style models, which evolve strictly in discrete steps.
 
Consider a system with state $e_t \in \mathbb{R}^D$ and exogenous signal $\mu_t \in \mathbb{R}^m$:
\begin{equation}
    \label{eq:general_system}
    e_{t+1} = g_\theta(e_t, \mu_t).
\end{equation}
Here $\mu_t$ collects every quantity other than $e_t$ that drives the update, such as the preceding hidden states and the next input. Equation~\eqref{eq:general_system} is therefore a representation of the dynamics, not a modelling assumption.
 
\begin{definition}[ISS-type Lyapunov function]
\label{def:iss-lyapunov}
A continuous function $V: \mathbb{R}^D \to \mathbb{R}_{\geq 0}$ is an \emph{ISS-type Lyapunov function} along a trajectory $(e_t)_{t \geq t_s}$ of~\eqref{eq:general_system} if there exist $\gamma \in (0,1)$ and $\tau \geq 0$ such that
\begin{equation}
    \label{eq:iss_working}
    V_{t+1} \leq \gamma V_t + \tau, \qquad V_t = V(e_t), \qquad t \geq t_s.
\end{equation}
\end{definition}
 
Iterating this inequality from the initial step $t_{s}$ yields
\begin{equation}
    \label{eq:iss_iterate}
    V_t \leq \frac{\tau}{1 - \gamma} + \gamma^{t - t_{s}} V_{t_{s}},
    \qquad t \geq t_{s}.
\end{equation}
The first term is a steady-state bound determined by $(\gamma, \tau)$. The second term is a transient that decays geometrically. Consequently, the sublevel set $\{V \leq \tau/(1-\gamma)\}$ is forward invariant, and every trajectory satisfying~\eqref{eq:iss_working} converges to it.
 
This explicit bound is central to our approach. We construct a Lyapunov function that measures transversal deviation. Whenever the contraction condition holds, the bound translates into a quantitative limit on the width of the trajectory tube.

\section{Principled Trajectory Regularization
via Lyapunov Tube Loss}
\label{sec:method}
We develop ControlJEPA in three steps. First, we decompose hidden-state displacement into longitudinal and transversal components (\S\ref{sec:method:intro}). Next, we define a dimensionless Lyapunov function that induces a tube via the ISS bound (\S\ref{sec:method:candidate}). Finally, we derive the Lyapunov Tube Loss by penalizing violations of the contraction condition and integrate it with the NTP objective (\S\ref{sec:method:loss}).
\subsection{Transversal Decomposition}
\label{sec:method:intro}

We instantiate the ISS framework of \S\ref{sec:prelim:maps} for transformer hidden-state trajectories. For each token position $t$ within a semantic segment, we define the \emph{transversal error} $e_t \in \mathbb{R}^D$ as the component of the hidden-state displacement that lies perpendicular to a reference direction connecting two semantic anchor states (formally derived in Eq.~\eqref{eq:decompose}). This error plays the role of the state in system~\eqref{eq:general_system}, and the Lyapunov candidate $V_t$ is constructed so that the decrease condition~\eqref{eq:iss_working} directly bounds the tube width.
\begin{figure}[t]
  \centering
  \includegraphics[width=\columnwidth]{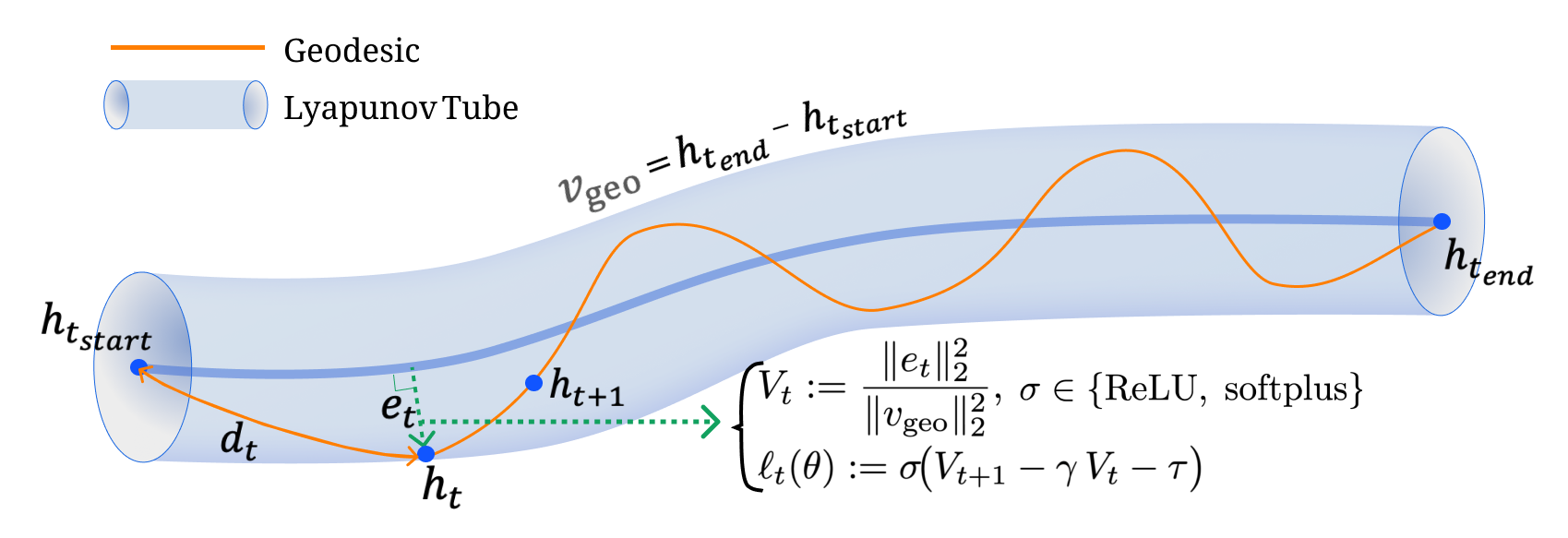}
  \vspace{-1em}
  \caption{\textbf{Lyapunov Tube.} The displacement $d_t = h_t - h_{t_{\text{start}}}$ is decomposed into a longitudinal component along the reference secant $v_{\text{geo}}$ and a transversal error $e_t \perp v_{\text{geo}}$. The Lyapunov Tube Loss penalizes violations of $V_{t+1} \leq \gamma V_t + \tau$, confining hidden-state trajectories within the tube. (Definition~\ref{def:lyapunov_tube}).}
  \label{fig:controlJEPA}
  \vspace{-1em}
\end{figure}
Let $h_t \in \mathbb{R}^D$ denote the hidden state at position $t$ and let $h_{t_{\mathrm{start}}}, h_{t_{\mathrm{end}}} \in \mathbb{R}^D$ be semantic anchor states. The \emph{reference secant} $v_{\mathrm{geo}}$ is the vector connecting these anchors; its unit vector $\hat{v}$ defines a preferred direction along which the trajectory is expected to evolve:
\begin{equation}
  \label{eq:vgeo}
  v_{\mathrm{geo}} = h_{t_{\mathrm{end}}}
    - h_{t_{\mathrm{start}}},
  \quad
  \hat{v} = \frac{v_{\mathrm{geo}}}
    {\|v_{\mathrm{geo}}\|_2},
  \quad
  d_t = h_t - h_{t_{\mathrm{start}}}.
\end{equation}
Here $d_t$ is the displacement of the intermediate hidden state $h_t$ from the start anchor. Projecting $d_t$ onto $\hat{v}$ decomposes it into a \emph{longitudinal} component (the progress along the reference direction) and a \emph{transversal} error $e_t$ (the deviation perpendicular to it):
\begin{equation}
  \label{eq:decompose}
  d_t = \underbrace{\langle d_t,\hat{v}\rangle\,
    \hat{v}}_{\text{longitudinal}}
    + \underbrace{e_t}_{\text{transversal}},
  \qquad
  e_t = (I - \hat{v}\hat{v}^\top) d_t,
  \qquad
  e_t \perp \hat{v}.
\end{equation}
The norm $\|e_t\|_2$ measures how far the trajectory deviates from the reference line at step $t$, and serves as the tube radius. We define the Lyapunov Tube in Definition~\ref{def:lyapunov_tube} below. Figure~\ref{fig:controlJEPA} illustrates this decomposition.

\subsection{Lyapunov Candidate}
\label{sec:method:candidate}

We now convert the transversal decomposition in Section~\ref{sec:method:intro} into a scalar stability certificate. 
The goal is to measure not only whether the trajectory is aligned with the reference direction, but also how far it deviates from the reference secant in magnitude. 
This scalar quantity will serve as the Lyapunov candidate whose ISS contraction yields an explicit tube bound around the semantic reference path.

A Lyapunov candidate is a nonnegative scalar function $V_t = V(e_t)$ that plays the role of an energy for the transversal deviation. We define the \emph{reference line} as the affine line $\{h_{t_{\mathrm{start}}} + \alpha\, v_{\mathrm{geo}} \mid \alpha \in \mathbb{R}\}$ spanned by the two semantic anchors, along which a perfectly straight trajectory would evolve.
In our setting, $V_t$ should satisfy three requirements: it should vanish exactly when the hidden state lies on the reference line, increase with the squared transversal deviation, and remain invariant under global rescaling of the hidden-state representation. To achieve this, we normalize the squared transversal error by the squared length of the reference secant.

\begin{boxeddefinition}{(Lyapunov Candidate).}{Lyapunov_candidate}
\label{def:Vt}
Given the transversal error $e_t$ and reference secant $v_{\mathrm{geo}}$, we define:
\begin{equation}
\label{eq:Vt}
    V_t 
    =
    \frac{\|e_t\|_2^2}{\|v_{\mathrm{geo}}\|_2^2}.
\end{equation}
\end{boxeddefinition}

This quantity is dimensionless and scale-invariant. 
Moreover, $V_t = 0$ if and only if $h_t$ lies on the reference line spanned by $v_{\mathrm{geo}}$. 
Thus, $V_t$ measures the relative transversal deviation of the hidden-state trajectory from the semantic reference path.

The normalization by $\|v_{\mathrm{geo}}\|_2^2$ is essential. 
Directly minimizing the unnormalized error $\|e_t\|_2^2$ would encourage the trajectory to collapse onto the reference line, which can reduce the rank of the hidden-state matrix and destroy representational capacity, as shown in Proposition~4.1. 
By contrast, the normalized candidate allows the model to reduce relative deviation either by controlling the transversal error or by increasing the semantic separation between the anchor states. 
This distinction is important because it permits geometric stabilization without forcing all hidden states into a low-dimensional geodesic subspace.

The ISS condition applied to this Lyapunov candidate takes the form
\begin{equation}
    \label{eq:iss_Vt}
    V_{t+1} \leq \gamma V_t + \tau,
    \qquad \gamma \in (0,1), \quad \tau \geq 0,
\end{equation}
where $\gamma$ controls the contraction rate and $\tau$ specifies the allowable residual transversal deviation. 
Applying the iterated bound~\eqref{eq:iss_iterate} with $t_s = t_{\mathrm{start}}$ gives
\begin{equation}
    V_t \leq \frac{\tau}{1-\gamma} 
    + \gamma^{t - t_{\mathrm{start}}} V_{t_{\mathrm{start}}}.
\end{equation}
In our setting, the start anchor satisfies $d_{t_{\mathrm{start}}} = h_{t_{\mathrm{start}}} - h_{t_{\mathrm{start}}} = 0$, hence $e_{t_{\mathrm{start}}} = 0$ and $V_{t_{\mathrm{start}}} = 0$. The transient term therefore vanishes identically, and the bound reduces to
\begin{equation}
\label{eq:steady_state}
    V_t \leq \frac{\tau}{1-\gamma}, \qquad 
    t \geq t_{\mathrm{start}}.
\end{equation}
This motivates the following geometric notion.

\begin{boxeddefinition}{(Lyapunov Tube).}{lyapunov_tube}
\label{def:lyapunov_tube}

Given contraction parameters $\gamma \in (0,1)$ and $\tau \geq 0$, the Lyapunov tube is the set of hidden states whose Lyapunov value lies within the ISS steady-state bound:
\begin{equation}
\label{eq:tube_def}
    \mathcal{B}_{\mathrm{tube}}
    =
    \left\{
        h_t
        \;\middle|\;
        V_t \leq \frac{\tau}{1-\gamma}
    \right\}.
\end{equation}

Equivalently, in terms of the transversal error, this corresponds to the tube-radius bound:
\begin{equation}
    \|e_t\|_2
    \leq
    r^\ast
    :=
    \sqrt{\frac{\tau}{1-\gamma}}\,
    \|v_{\mathrm{geo}}\|_2 .
\end{equation}
\end{boxeddefinition}

The tube $\mathcal{B}_{\mathrm{tube}}$ serves as the target attractor set of the regularization: when the ISS condition~\eqref{eq:iss_Vt} holds at every step, all hidden states are guaranteed to remain inside $\mathcal{B}_{\mathrm{tube}}$; the Lyapunov Tube Loss (\S~\ref{sec:method:loss}) is designed precisely to penalize violations of this condition. In practice, however, the ISS condition need not hold at every step to achieve better trade-offs between the model's stability and capacity. Section~\ref{sec:theoretical_analysis} analyses this partial-ISS regime and characterizes what our method achieves when the condition holds only approximately.

Therefore, enforcing the ISS contraction condition on $V_t$ provides an explicit and interpretable geometric guarantee: the hidden-state trajectory is confined to a tube around the reference secant whose width is determined by the contraction rate $\gamma$, the slack parameter $\tau$, and the reference scale $\|v_{\mathrm{geo}}\|_2$. 
Unlike angle-based penalties, this construction controls the magnitude of transversal deviation while preserving a nonzero slack region for expressive hidden-state representations.

\subsection{ControlJEPA: Stable Trajectory Regularization}
\label{sec:method:loss}

Violations of $V_{t+1} - \gamma V_t - \tau \leq 0$ are penalized through a non-negative activation $\sigma$:
\begin{equation}
  \label{eq:pointwise}
  \ell_t(\theta)
    = \sigma\!\bigl(
      V_{t+1} - \gamma V_t - \tau\bigr),
  \qquad
  \sigma \in \{\mathrm{ReLU},\;\mathrm{softplus}\},
\end{equation}
so $\ell_t = 0$ whenever the ISS condition is already satisfied ($\sigma = \mathrm{ReLU}$) or near zero ($\sigma = \mathrm{softplus}$). Following \citet{huang2026semantic}, we restrict penalization to valid transitions
$\mathcal{T} \subseteq \{t_{\mathrm{start}}{+}1,
\ldots, t_{\mathrm{end}}{-}1\}$ within contiguous
semantic segments. We define the \emph{Lyapunov Tube Loss} as follows:
\begin{boxeddefinition}{(Lyapunov Tube Loss).}{lyapunov_tube_loss}
  \label{def:tube_loss}
  \begin{equation}
    \label{eq:tube_loss}
    \mathcal{L}_{\mathrm{Lyap}}(\theta)
      = \frac{1}{|\mathcal{T}|}
      \sum_{t \in \mathcal{T}}
        \sigma\!\Biggl(
          \frac{\|e_{t+1}\|_2^2}
               {\|v_{\mathrm{geo}}\|_2^2}
          - \gamma
          \frac{\|e_t\|_2^2}
               {\|v_{\mathrm{geo}}\|_2^2}
          - \tau
        \Biggr),
  \end{equation}
  with $\gamma \in (0,1)$ (contraction rate) and
  $\tau \geq 0$ (orthogonal slack).

\end{boxeddefinition}
\noindent
The \textbf{ControlJEPA} objective augments NTP with the Lyapunov Tube Loss:
\begin{equation}
  \label{eq:controljepa}
  \mathcal{L}_{\mathrm{controlJEPA}}(\theta)
    = \mathcal{L}_{\mathrm{NTP}}(\theta)
    + \lambda\,\mathcal{L}_{\mathrm{Lyap}}(\theta),
\end{equation}
where $\lambda > 0$ controls the regularization strength. The only additions to standard fine-tuning are the projection in Eq.~\eqref{eq:decompose} and the scalar comparisons in Eq.~\eqref{eq:pointwise}.

The geometric effects of $\mathcal{L}_{\mathrm{Lyap}}$ are characterized formally in Section~\ref{sec:theoretical_analysis}: minimizing the loss provably confines the mean transversal level (Theorem~\ref{thm:transversal}, bound~\eqref{eq:mean-bound}) and suppresses its temporal variance (bound~\eqref{eq:var-bound}), yielding representations that are both geometrically regular and informationally rich. Crucially, this operates in a \emph{partial-ISS} regime where $\mathcal{L}_{\mathrm{Lyap}} > 0$, preserving the representational capacity that NTP requires (Section~\ref{sec:partial-iss}).

\section{Theoretical Analysis}
\label{sec:theoretical_analysis}
{\bf Scope of the theoretical analysis.}
The bounds in this section hold for trajectories on the training distribution; no formal guarantee extends to unseen inputs. Whether the induced geometric structure generalizes is an empirical question addressed in Section~\ref{sec:experiments}. Figure~\ref{fig:infer_time} shows the stable transversal profile persists at inference time, suggesting training reshapes the representation manifold globally. A theoretical characterization of this generalization remains an open question.
The Lyapunov Tube Loss is designed to enforce the geodesic hypothesis~\citep{huang2026semantic} without destroying the representational capacity that next-token prediction requires. We establish two preliminary results that jointly motivate why the method must operate in a \emph{partial-ISS regime} (\S\ref{sec:capacity}--\ref{sec:partial-iss}), and then state the central theorem characterizing what minimizing $\mathcal{L}_{\mathrm{Lyap}}$ achieves in that regime (\S\ref{sec:theorem}).
\subsection{Why Minimising \texorpdfstring{$\|e_t\|_2$}{the transversal error} Collapses Model Capacity}
\label{sec:capacity}
To see why the transversal error $\|e_t\|_2$ must be \emph{controlled} rather than minimized, consider the effect on the model's output distribution. Stack the hidden states as rows of $H\in\mathbb{R}^{T\times D}$. Using the orthogonal decomposition $h_t - h_0 = \alpha_t\hat{v}\|v_{\mathrm{geo}}\|_2 + e_t$:
\begin{equation}
  H \;=\; \mathbf{1}h_0^\top + \alpha\bigl(\hat{v}\|v_{\mathrm{geo}}\|_2\bigr)^\top + E,
  \label{eq:H_decomp}
\end{equation}
where $\alpha = (\alpha_1,\ldots,\alpha_T)^\top \in \mathbb{R}^T$ and $E\in\mathbb{R}^{T\times D}$ has rows $e_t$. The logit matrix then is $L=HU^\top$ (where $U$ is the unembedding matrix). We define the rank-two logits:
\begin{equation}
   \Tilde{H} \;:=\;\mathbf{1}h_0^\top + \alpha\bigl(\hat{v}\|v_{\mathrm{geo}}\|_2\bigr)^\top, \quad \Tilde{L} \;:=\; \Tilde{H}U^\top 
   \label{eq:rank-2-logit}
\end{equation}
Then $\operatorname{rank}(\Tilde{L}) \leq 2$, since $\Tilde{H}$ is a sum of two rank-one terms and right multiplication cannot increase rank.

The object is to see how far the next-token distribution $p_\theta(.|x_{<t})_{t=1}^T$ can deviate from a two-parameter family, as a function of the transversal error bound $\delta$. Following that, we come up with proposition \ref{prop:capacity}, which claims that such deviation is bounded by an explicit, continuous, monotone function of $\delta$.

\begin{restatable}[Spectral capacity bound]{proposition}{Capacity}
\label{prop:capacity}
Suppose $\|e_t\|_2 \le \delta$ for all $t$. Then 
\begin{equation}
    \frac{1}{T} \sum_{j > 2} \sigma_j\,(L)^2 \leq \delta^2\,|U|_{op}^2,
    \label{eq:singular-bound}
\end{equation}
and consequently there is a smooth manifold $\mathcal{N}$ of dimension at most $2$ in the probability simplex with
\begin{equation}
     \frac{1}{T} \sum_{t=1}^{T} \operatorname{dist}_{\ell_2}(p_\theta(.|x_{<t}), \mathcal{N})^2 \leq \frac{1}{4}\delta^2\,|U|_{op}^2. 
    \label{eq:prob-dist}
\end{equation}
\end{restatable}


The formal proof is deferred to Appendix~\ref{app:capacity}. The key message is that any regulariser that drives $\|e_t\|_2$ too small destroys representational capacity as a side effect, regardless of how well it aligns the trajectory with the geodesic. One thing worth flagging is that the bound involves $|U|_{op}$, so it decreases with $\delta$ only if the unembedding norm does not grow to compensate. We therefore report measured $|U|_{op}$ along side $\delta$ across $\lambda_{\mathrm{Lyap}}$ sweep in Figure~\ref{fig:U_op}(a), so that the product $\delta^2\,|U|_{op}^2$ governing \eqref{eq:prob-dist} can be read off directly rather than assumed.

\begin{figure}[t]
    \centering
    \includegraphics[width=0.95\columnwidth]{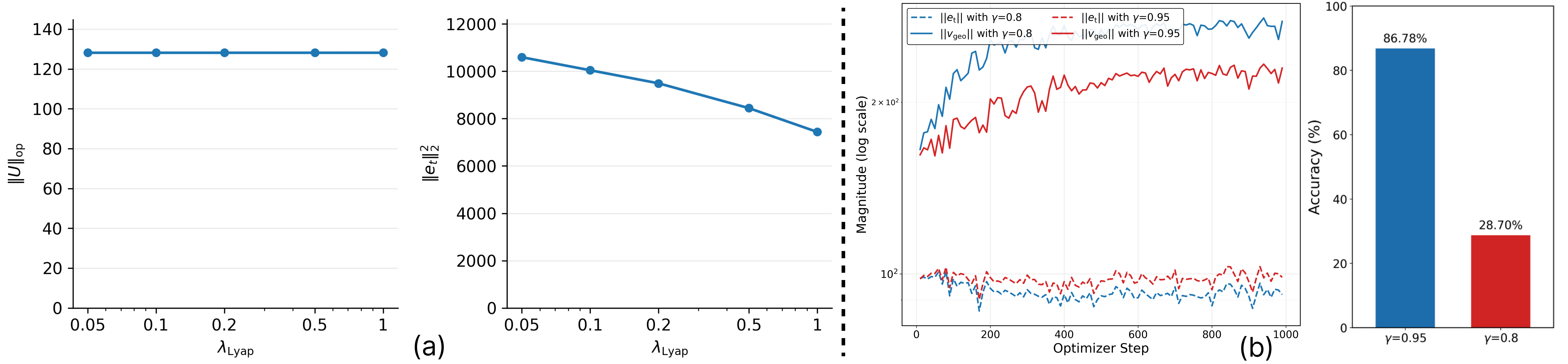}
    \vspace{-0.5em}
    \caption{(a) Unembedding norm and transversal energy across $\lambda_{\mathrm{Lyap}}$. (b) Impact of contraction rate $\gamma$ on geometric components and downstream accuracy.}
    \label{fig:U_op}
    \vspace{-1em}
\end{figure}

\subsection{The Partial-ISS Regime}
\label{sec:partial-iss}

Full ISS satisfaction leads to the same potential destruction. If $\mathcal{L}_{\mathrm{Lyap}}(\theta^*) \approx 0$, the ISS condition holds at every step, and iterating from $V_0 = 0$ gives:
\begin{equation}
  V_t \;\le\; \frac{\tau}{1-\gamma},
  \label{eq:iss_iter}
\end{equation}
so $\|e_t\|_2 \le \sqrt{\tau/(1-\gamma)}\,\|v_{\mathrm{geo}}\|_2$. For small $\tau$ this again forces $\|e_t\|_2$ small, violating Proposition~\ref{prop:capacity}. Hence, minimizing $\mathcal{L}_{\mathrm{Lyap}}$ closely to zero is likely to be  \emph{incompatible} with good NTP performance.

The practical training regime therefore necessarily satisfies $\mathcal{L}_{\mathrm{Lyap}}(\theta) \gg 0$ (Figure~\ref{fig:de_snr_lhp}(b)): the ISS condition is violated at some steps, keeping $\|e_t\|_2$ strictly positive and preserving the representational capacity stored in $E$. We call this the \emph{partial-ISS regime}. Figure~\ref{fig:U_op}(b) confirms empirically that well-performing checkpoints always reside in this regime.

The question then becomes: what does adding Lyapunov Tube Loss actually achieve in the partial-ISS regime? The answer is the stability of the \emph{temporal profile} of $V_t$ rather than suppressing its level, which is the content of the next theorem. Specifically, we measure profile stability by the temporal variance $\sigma^2 := \operatorname{Var}_t[V_t]$; a small $\sigma^2$ means the trajectory maintains a \emph{uniform} distance from the geodesic across all token positions, rather than alternating between near-collapse and large excursions.

\subsection{Transversal Profile Characterization}
\label{sec:theorem}

We first fix notation for the partial-ISS regime. At each step, define the \textbf{signed innovation}
\begin{equation}
  \varepsilon_t \;:=\; V_{t+1} - \gamma V_t - \tau,
  \label{eq:innovation}
\end{equation}
so that $\varepsilon_t \le 0$ when the ISS condition is satisfied and $\varepsilon_t > 0$ when it is violated.
The Lyapunov Tube Loss with the $\operatorname{softplus}$ activation satisfies
\begin{equation}
  \mathcal{L}_{\mathrm{Lyap}}(\theta)\;=\;\frac{1}{|\mathcal{T}|}\sum_{t\in\mathcal{T}}    \log\!\bigl(1+e^{\varepsilon_t}\bigr).
  \label{eq:loss-innovations}
\end{equation}
Let $M:=\max_{t\in\mathcal T}\max(\varepsilon_t,0)$ denote the largest ISS violation along the trajectory, and $n:=\frac{M^2}{\log(1+e^{M})}$, so that $\frac{1}{|\mathcal T|}\sum_{t\in\mathcal T}\max(\varepsilon_t,0)^2 \le n\,\mathcal L_{\mathrm{Lyap}}(\theta)$ (Appendix~\ref{app:transversal}).

\begin{restatable}[Transversal Profile Characterization]{theorem}{Tranversal}
\label{thm:transversal}
Let $\{V_t\}_{t=0}^{T}$ be the Lyapunov sequence with $V_0 = 0$, and let $\{\varepsilon_t\}$ be the innovations defined in \eqref{eq:innovation}. Assume $\gamma\in(0,1)$ and $\tau\ge 0$. Then the following hold for every trajectory, without any distributional assumption on $\{\varepsilon_t\}$:

\noindent\textbf{(i) Mean confinement.}
The time-averaged Lyapunov value satisfies
\begin{equation}
  \bar{V}_T \;\le\;\frac{\tau}{1-\gamma}\;+\; \frac{\mathcal{L}_{\mathrm{Lyap}}(\theta)}{1-\gamma}.
  \label{eq:mean-bound}
\end{equation}
\noindent\textbf{(ii) Variance control.}
The temporal variance of the Lyapunov sequence satisfies
\begin{equation}
  \operatorname{Var}_t[V_t] \;\le\; \frac{2n}{(1-\gamma)^2}\left[1+\frac{4\gamma^2}{(1-\gamma)^2}\right]\mathcal{L}_{\mathrm{Lyap}}(\theta)
  \;+\;\frac{4\tau^2}{(1-\gamma)^2}\left[1+\frac{2\gamma^2}{(1-\gamma)^2}\right]
  \;+\;\mathcal{R}_T,
  \label{eq:var-bound}
\end{equation}
where $\mathcal{R}_T$ is an explicit boundary remainder with $\mathcal{R}_T=O(\frac{1}{T})$ whenever $M$ is bounded uniformly in $T$.

\noindent\textbf{(iii) Practical tube width.}
By Jensen's inequality,
\begin{equation}
  \mathbb{E}_t\!\left[\|e_t\|_2\right] \;\le\; \sqrt{\bar{V}_T}\cdot\|v_{\mathrm{geo}}\|_2 \;\le\;  \|v_{\mathrm{geo}}\|_2 \cdot\sqrt{\frac{\tau + \mathcal{L}_{\mathrm{Lyap}}(\theta)}{1-\gamma}}.
  \label{eq:tube-width}
\end{equation}
\end{restatable}
The full proof is in Appendix~\ref{app:transversal}.  Together, \eqref{eq:mean-bound} and \eqref{eq:var-bound} relate the mean transversal level and its temporal fluctuations to the Lyapunov Tube Loss, the partial-ISS parameters $(\gamma,\tau)$, and the trajectory-dependent violation factor $n$. Under additional uniform control of the maximal positive innovation $M$, the variance bound also yields a horizon-independent control of the transversal fluctuations.

\section{Experimental Results}
\label{sec:experiments}
In this section, we empirically validate ControlJEPA against standard fine-tuning and geodesic regularizers, showing that ISS-based trajectory control improves generalization, sample efficiency, and structural stability beyond heuristic geometric methods.

\subsection{Experimental Setup}
\label{setup}



We evaluate ControlJEPA following Semantic Tube Prediction~\cite{huang2026semantic} across diverse tasks, architectures, and scales. Benchmarks include NL-RX-SYNTH, NL-RX-TURK~\cite{locascio}, GSM8K~\cite{cobbe2021trainingverifierssolvemath}, Spider~\cite{yu-etal-2018-spider}, NQ-Open~\cite{lee-etal-2019-latent}, and HellaSwag~\cite{zellers-etal-2019-hellaswag}. Evaluated models span Llama-3.2 (1B/3B/8B)~\cite{grattafiori2024llama3herdmodels}, Gemma-2-2b-it~\cite{gemmateam2024gemma2improvingopen}, OpenELM-1.1B~\cite{mehta2024openelmefficientlanguagemodel}, OLMo-2-1B~\cite{olmo20252olmo2furious}, Qwen3-1.7B~\cite{yang2025qwen3technicalreport}, and DeepSeek-R1-Distill-Qwen-1.5B~\cite{Guo_2025}. We compare against standard fine-tuning, LLM-JEPA~\cite{huang2025llm}, and Semantic Tube Prediction~\cite{huang2026semantic}. All methods share an identical codebase, splits, decoding, validation rules, and 5 random seeds (82, 23, 37, 84, 4). Baselines use their published, dataset-tuned optimal epochs and learning rates; ControlJEPA fixes these shared hyperparameters and tunes only its stability parameters (Lyapunov penalty $\lambda$, contraction rate $\gamma$, ISS slack $\tau$), conservatively favoring the baselines in tuning budget. All experiments ran on two NVIDIA H100 80GB GPUs.



\subsection{Accuracy and Generalization}
\begin{figure}[t]
    \centering
    \includegraphics[width=0.9\columnwidth]{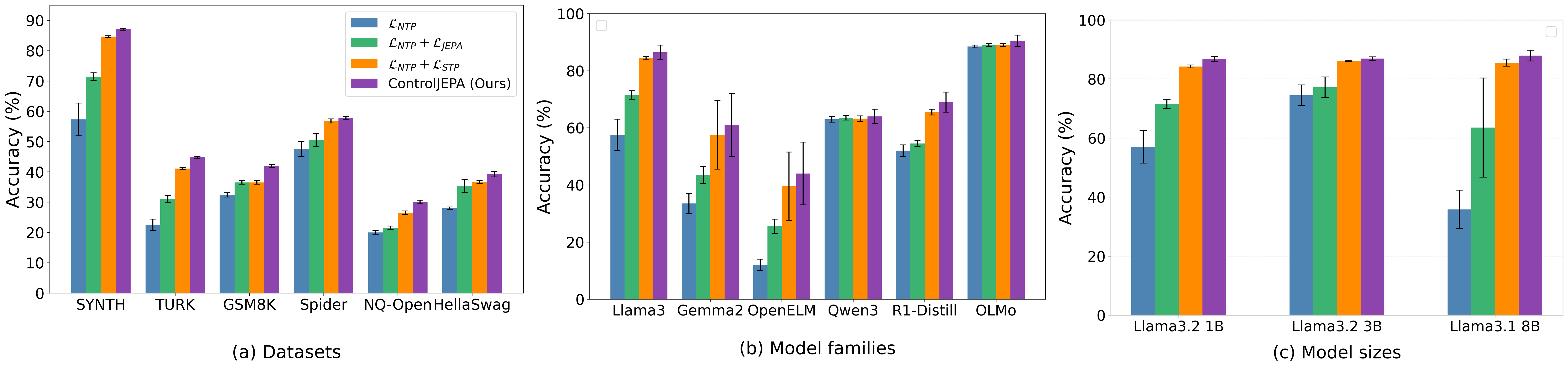}
    \vspace{-0.5em}
    \caption{Accuracy and generalization performance of $\mathcal{L}_{NTP}$, LLM-JEPA, STP, and ControlJEPA across (a) Datasets, (b) Model families, and (c) Model sizes.}
    \label{fig:acc}
    \vspace{-1em}
\end{figure}
As shown in Figure~\ref{fig:acc}, ControlJEPA consistently outperforms all baselines across datasets, architectures, and model scales. Its gains are especially pronounced on complex reasoning tasks such as GSM8K and NQ-Open, where heuristic methods like STP suffer from unconstrained error magnitudes. By enforcing the Lyapunov condition $V_{t+1} \le \gamma V_t + \tau$, ControlJEPA bounds transversal drift and keeps hidden states anchored to the semantic geodesic over long contexts. Its latent-manifold formulation also generalizes across Llama3, Gemma2, OpenELM, Qwen3, R1-Distill, and OLMo, and scales smoothly across Llama-3 models from 1B to 8B, with the dimensionless Lyapunov candidate preventing both scale collapse and representational bottlenecks.


\subsection{Data Efficiency and Signal-to-Noise Ratio}


\begin{figure}[t]
    \centering
    \includegraphics[width=1\columnwidth]{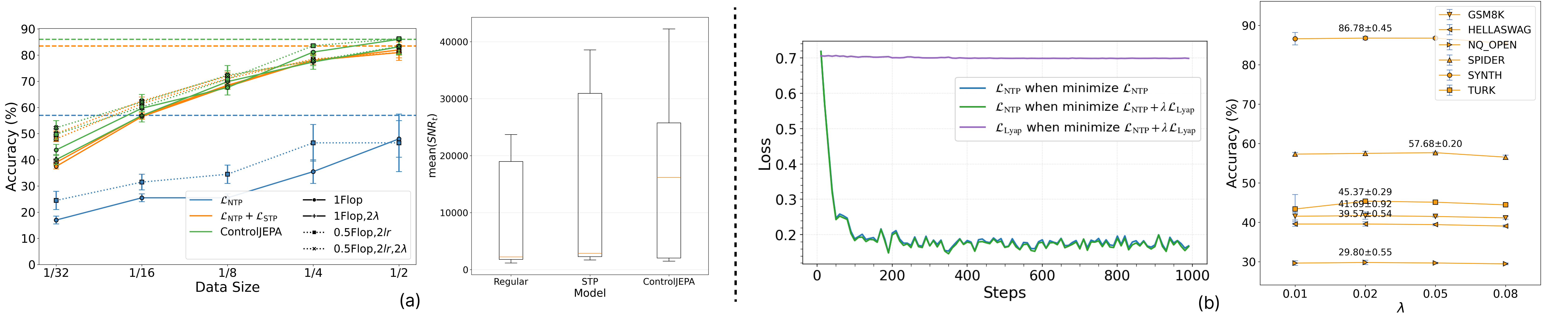}
    \caption{(a) Data Efficiency and Signal-to-Noise ratio. (b) Training Dynamics and Hyperparameter Robustness.}
    \label{fig:de_snr_lhp}
\end{figure}


A key benefit of trajectory noise filtering is improved sample efficiency. As shown in Figure~\ref{fig:de_snr_lhp}(a) (Left), ControlJEPA outperforms standard fine-tuning and STP in moderate data regimes and remains competitive under severe subsampling, matching full-dataset $\mathcal{L}_{NTP}$ performance with only a fraction of the data. This efficiency is consistent with an improved signal-to-noise ratio: by constraining orthogonal deviations through the ISS condition with slack $\tau$, ControlJEPA can extract semantic alignment signals more effectively. Using the previously defined displacement decomposition, we quantify this effect on open-loop greedy generations as
\begin{equation}
\mathrm{SNR}_t
=
\frac{|\langle d_t, \hat{v} \rangle \hat{v}|_2^2}
{||e_t||_2^2}
=
\frac{\langle d_t, \hat{v} \rangle^2}
{||e_t||_2^2},
\label{eq}
\end{equation}
which measures the ratio of aligned displacement energy to orthogonal residual energy. Figure~\ref{fig:de_snr_lhp}(a) (Right) supports this interpretation, showing that ControlJEPA achieves a substantially higher median SNR than both baselines.

\subsection{Training Dynamics and Hyperparameter Sensitivity}
\label{hyper}
As shown in Figure~\ref{fig:de_snr_lhp}(b) (Left), ControlJEPA is both optimization-stable and robust to hyperparameters. The $\mathcal{L}{\mathrm{NTP}}$ curves overlap when trained alone or jointly with ControlJEPA, indicating that the geometric constraint improves trajectory structure without harming language modeling. The Lyapunov loss stabilizes near $0.69$, matching the softplus minimum $\ln(1+e^0)\approx0.693$ once the ISS condition $V_{t+1}-\gamma V_t-\tau\le0$ is satisfied. Moreover, ControlJEPA maintains consistently high accuracy across $\lambda\in[0.01,0.08]$, providing a broad stability region and reducing the need for exhaustive tuning as show in Figure~\ref{fig:de_snr_lhp}(b) (Right).


\section{Empirical Analysis}
\label{sec:analysis}
\subsection{Implicit Tube Preservation at Inference Time}


At generation time, the model operates open-loop, with the Lyapunov loss deactivated. To test whether the learned geometry generalizes, Figure~\ref{fig:infer_time}(a) plots the unnormalized squared transversal error (Left) and its variance (Right) on unseen trajectories. Standard fine-tuning shows substantial drift, while STP reduces the mean error but remains highly oscillatory. In contrast, ControlJEPA achieves the lowest and most stable transversal error, keeping hidden states tightly confined within the semantic tube even without an active penalty. This suggests that ISS-based training reshapes the representation manifold rather than merely memorizing training trajectories.

\begin{figure}[t]
    \centering    \includegraphics[width=1\columnwidth]{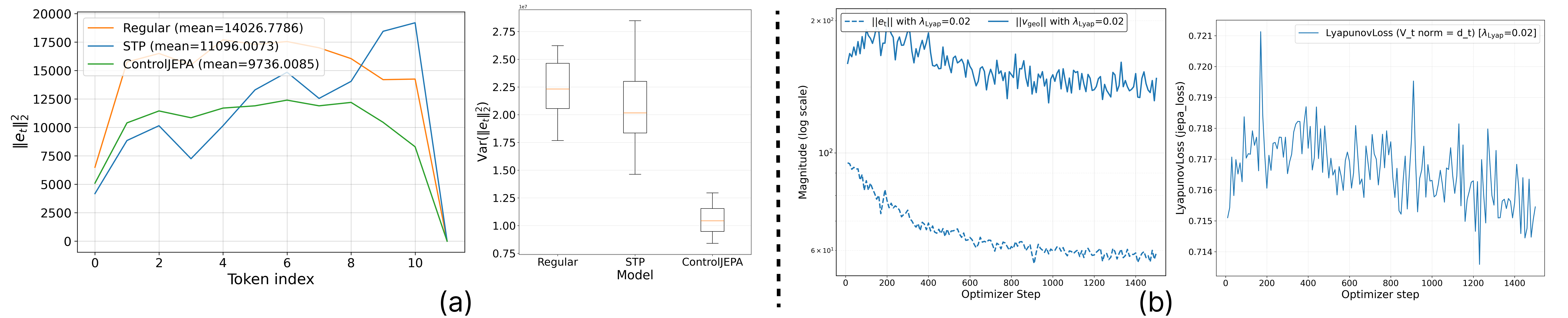}
    \vspace{-1.5em}
    \caption{(a) Inference-time trajectory of squared transversal error $\|e_t\|_2^2$ across generated tokens in NL-RX-SYNTH dataset. (b) Without normalization, the representation scale collapses, shrinking both $\|e_t\|$ and $\|v_{\mathrm{geo}}\|$ decrease over time. \textbf{(Right)} Normalizing by local displacement introduces severe heteroscedastic instability and volatile loss dynamics.}
    \label{fig:infer_time}
    \vspace{-1em}
\end{figure}



\subsection{Preventing Dimensionality Collapse and Heteroscedastic Instability}

The dimensionless Lyapunov candidate
$V_t = \frac{\|e_t\|_2^2}{\|v_{\mathrm{geo}}\|_2^2}$
is designed to avoid two key failure modes in trajectory optimization: scale collapse and heteroscedastic instability.

\textbf{Preventing Scale Collapse.} Penalizing the unnormalized transversal energy, $V_t=|e_t|_2^2$, encourages a trivial solution that shrinks the entire representation space, with $d_t\to0$ and $v_{\mathrm{geo}}\to0$. As shown in Figure~\ref{fig:infer_time}(b) (Left), this collapses the chord length $|v_{\mathrm{geo}}|_2^2$, destroys manifold capacity, and reduces accuracy to 6.3\%. Normalizing by a length scale makes $V_t$ scale-invariant and prevents this collapse.


\textbf{Preventing Heteroscedastic Instability.} Normalizing by the local displacement, $V_t = \frac{\|e_t\|_2^2}{\|d_t\|_2^2}$, destabilizes the ISS condition because $\|d_t\|_2^2$ grows over time, creating a moving denominator that assigns inconsistent penalties across positions. This heteroscedasticity leads to volatile loss dynamics and suboptimal convergence, as shown in Figure~\ref{fig:infer_time}(b) (Right). In contrast, the global chord length $\|v_{\mathrm{geo}}\|_2^2$ provides a sequence-level invariant for uniformly measuring Lyapunov violations.

\subsection{The Necessity of Orthogonal Slack and Joint Equilibrium}
To validate Proposition~\ref{prop:capacity} and the Partial-ISS Regime discussed in Section~\ref{sec:partial-iss}, 
we ablate the contraction rate $\gamma$ and 
track the orthogonal 
slack $|e_t|$ and macroscopic chord $|v_{\mathrm{geo}}|$ during training with $\lambda_{\mathrm{Lyap}}$ = 0.05, as shown in Figure~\ref{fig:U_op}(b).

{\bf Capacity Collapse under Strict Constraints.} With a strong contraction, $\gamma=0.8$, ControlJEPA approximates the strict geodesic constraint of STP, driving $|e_t|$ toward zero. 
As predicted by Proposition~\ref{prop:capacity}, removing this orthogonal slack collapses the hidden-state rank, restricts expressivity, and reduces downstream accuracy to 28.7\%, exposing the limitation of static geodesic interpolation.\\
{\bf Geometric Linearisation via Chord Expansion.} Under optimal soft regularization, $\gamma=0.95$, the system converges to the joint equilibrium and stabilizes at a positive value, preserving representational capacity and yielding 86.78\% accuracy, while $|v_{\mathrm{geo}}|$ grows substantially. This confirms that ControlJEPA achieves geometric linearization not by eliminating microscopic slack, but by expanding the global reference frame under the NTP objective.






\section{Related Work}
\label{sec:related_work}

{\bf Joint-Embedding Predictive Architectures.} JEPA~\cite{lecun2022path} learns representations by predicting target embeddings from context, and has been extended to vision~\citep{assran2023self,bardes2024v,assran2025v,mur2026v,hu20243d,nam2026causaljepalearningworldmodels}, language and multimodal learning~\citep{huang2025llm,chen2025vl}, sequential and signal-based settings~\citep{thimonier2024t,ghaemi2025seq,guetschel2024s}, and WorldModel and energy-based models~\citep{terver2026lightweight, maes2026leworldmodelstableendtoendjointembedding}. While many methods address dimensional collapse through embedding-space regularization~\citep{jing2021understanding,bardes2021vicreg,balestriero2511lejepa,hartman2025sparsejepa,kuang2026rectified,mo2024connecting}, STP~\citep{huang2026semantic} instead regularizes hidden-state trajectories via cosine-based geodesic alignment. ControlJEPA replaces cosine alignment with a one-sided, scale-invariant Lyapunov contraction constraint on transversal deviation, avoiding collapse through partial enforcement and normalization.

{\bf Geometric Regularization of Hidden-State Trajectories.} Trajectory straightening has been observed in both vision~\citep{henaff2019perceptual} and language~\citep{hosseini2023large}, motivating methods such as Temporal Straightening~\citep{wang2026temporal} and STP~\citep{huang2026semantic}. However, these cosine-based regularizers control only the \emph{direction} of motion, not the \emph{magnitude} of deviation. ControlJEPA penalizes violations of $V_{t+1}\le\gamma V_t+\tau$, bounding deviation where satisfied and its mean otherwise.

{\bf Control-Theoretic Methods for Deep Learning.} Control-theoretic methods enforce stability through either architectural design~\citep{haber2018stable,nguyen2023mitigating,nguyen2024pidformer,abdullaev2025transformer,nguyen2026activation,nguyen2026muonssm} or auxiliary losses~\citep{rodriguez2022lyanet,tsukamoto2020neural}. ControlJEPA differs in two ways: it uses a discrete-time practical-ISS contraction condition suited to step-wise hidden-state dynamics, and enforces it as a soft trajectory-level penalty rather than an architectural constraint.

\section{Concluding Remarks}
\label{sec:conclusion}


We presented ControlJEPA, a trajectory regularization framework built on a discrete-time ISS-type contraction condition. Its Lyapunov Tube Loss penalizes violations of this condition on normalized transversal deviation. Wherever the condition holds, it implies an explicit tube-width bound. In the partial-ISS regime used in practice, the effective width is governed by the penalty strength and the contraction rate. Theoretical analysis identifies the partial-ISS regime as the balance point between geometric regularity and representational capacity. Empirically, ControlJEPA consistently improves accuracy and data efficiency over standard fine-tuning and prior trajectory regularizers across six benchmarks and six model families. A smoother transversal profile also persists at inference time, suggesting that the method reshapes the representation manifold rather than memorizing trajectories.

{\bf Limitations.} Our guarantees apply only to trajectories observed during training and do not extend to out-of-distribution settings. Inference-time stability remains empirical, and experiments are limited to fine-tuning models up to 8B parameters. In addition, semantic anchors currently depend on dataset-provided segmentation.

{\bf Future directions.} Future work includes extending the framework to vision and multimodal models, studying the interaction between $\mathcal{L}_{\mathrm{NTP}}$ and $\mathcal{L}_{\mathrm{Lyap}}$, and incorporating Lyapunov constraints directly into decoding for inference-time geometric control.
\medskip

\begin{ack}
This paper is dedicated to Professor Richard G. Baraniuk on the occasion of his 60th birthday.

This research is supported by the National Research Foundation Singapore under the AI Singapore Programme (AISG Award No: AISG2-TC-2023-012-SGIL). This research is also supported by the Ministry of Education, Singapore, under the Academic Research Fund Tier 1 (FY2023) (A-8002040-00-00, A-8002039-00-00), the NUS Presidential Young Professorship Award (A-0009807-01-00), the NUS Artificial Intelligence Institute--Seed Funding (A-8003062-00-00), and the Cross Faculty Grant 2025, CFG25-012 (A-8004460-00-00).
\end{ack}

\bibliographystyle{plainnat}
\bibliography{biblio}




\newpage 
\appendix
\begin{center}
{\bf \Large{Supplement to ``ControlJEPA: Principled Trajectory Regularization via Lyapunov Tube Loss''}}
\end{center}

\DoToC

\section{Theoretical Proofs and Supplementary Analysis}

\subsection{Proof of Proposition~\ref{prop:capacity}}
\label{app:capacity}

\subsubsection{Set up and notations}
Throughout, $\|\cdot\|_2$ denotes the Euclidean norm on vectors, $\|\cdot\|_F$ the Frobenius norm, $\|A\|_{\mathrm{op}}=\sigma_1(A)$ the operator norm, and $\sigma_1(A)\ge\sigma_2(A)\ge\cdots$ the singular values in decreasing order, indexed from one.

Let $h_t\in\mathbb R^{D}$ denote the hidden state at position $t$ after the final normalisation and before the output projection, let $U\in\mathbb R^{|\mathcal V|\times D}$ be the unembedding matrix, and let $H\in\mathbb R^{T\times D}$ have rows $h_t^\top$, so that the logit matrix is $L=HU^\top$ with rows $L_t=Uh_t$ and
$p_\theta(\cdot\mid x_{<t})=\operatorname{softmax}(L_t)$.

Using the orthogonal decomposition $h_t-h_0=\alpha_t\hat v\|v_{\text{geo}}\|_2+e_t$ with $e_t\perp\hat v$, set
\[
  \widetilde H:=\mathbf 1h_0^\top+\alpha\big(\hat v\|v_{\text{geo}}\|_2\big)^\top,\qquad \widetilde L:=\widetilde HU^\top,\qquad E:=H-\widetilde H,
\]
so $E$ has rows $e_t^\top$ and $L-\widetilde L=EU^\top$. Write $\mathcal S:=\operatorname{span} h_0,v_{\text{geo}})$, so $\dim\mathcal S\le2$ and every row $\tilde h_t=h_0+\alpha_tv_{\text{geo}}$ of $\widetilde H$ lies in $\mathcal S$.

\subsubsection{Two lemmas}

\begin{lemma}[Rank of the reference logits]
  \label{lem:rank}
  $\operatorname{rank}(\widetilde L)\le2$.
\end{lemma}
\begin{proof}
  $\widetilde H=\mathbf 1h_0^\top+\alpha v_{\text{geo}}^\top$ is a sum of two outer products, each of rank at most one, so $\operatorname{rank}(\widetilde H)\le2$.  Since $\operatorname{rank} AB)\le\min\ \operatorname{rank}(A),\operatorname{rank}(B)\}$, right multiplication by $U^\top$ cannot increase rank.
\end{proof}

\begin{lemma}[Softmax is $\tfrac12$-Lipschitz in $\ell_2$]
  \label{lem:lip}
  For all $z,z'\in\mathbb R^{|\mathcal V|}$,
  $\|\operatorname{softmax}(z)-\operatorname{softmax}(z')\|_2\le\tfrac12\|z-z'\|_2$.
\end{lemma}
\begin{proof}
Write $p=\operatorname{softmax}(z)$.  Differentiating $p_i=e^{z_i}/\sum_je^{z_j}$ gives the Jacobian $J(z)=\operatorname{diag}(p)-pp^\top$, which is symmetric. For any $x$,
  \[
    x^\top J(z)x=\sum_ip_ix_i^2-\Big(\sum_ip_ix_i\Big)^2=\operatorname{Var}_p(X),
  \]
where $X$ takes the value $x_i$ with probability $p_i$.  Hence $J(z)$ is positive semidefinite and, by symmetry and the variational characterisation of the largest eigenvalue, $\|J(z)\|_{\mathrm{op}}=\lambda_{\max}(J(z)) =\sup_{\|x\|_2=1}\operatorname{Var}_p(X)$.  Since $X$ takes values in  $[x_{\min},x_{\max}]$, Popoviciu's inequality gives $\operatorname{Var}_p(X)\le(x_{\max}-x_{\min})^2/4$, and $(x_{\max}- x_{\min})^2\le2(x_{\max}^2+x_{\min}^2)\le2\|x\|_2^2$ because the extremes are attained at distinct coordinates. Therefore $\|J(z)\|_{\mathrm{op}}\le\tfrac12$ for every $z$.

The domain $\mathbb R^{|\mathcal V|}$ is convex, so the segment $\gamma(s)=z'+s(z-z')$, $s\in[0,1]$, lies in it.  By the chain rule and the fundamental theorem of calculus, $\operatorname{softmax}(z)-\operatorname{softmax}(z')   =\int_0^1J(\gamma(s))(z-z')\,ds$, and taking norms,
  \[
    \big\|\operatorname{softmax}(z)-\operatorname{softmax}(z')\big\|_2   \le\Big(\sup_{s\in[0,1]}\|J(\gamma(s))\|_{\mathrm{op}}\Big)\|z-z'\|_2
    \le\tfrac12\|z-z'\|_2 . \qedhere
  \]
\end{proof}

\subsubsection{Proof of Proposition~\ref{prop:capacity}}
\Capacity*
\begin{proof}
By the Eckart-Young-Mirsky theorem, the squared spectral tail of $L$ equals its best rank-two approximation error, and $\Tilde{L}$ is one admissible competitor:
\begin{equation}
    \label{eq:sum-sing}
    \sum_{j>2}\sigma_j(L)^2 = \min_{\operatorname{rank}(R) \le 2} |L - R|_F^2 \leq |L - \Tilde{L}|_F^2 =|EU^\top|_F^2.
\end{equation}
By submultiplicativity, 
\begin{equation}
|EU^\top|_F^2\leq |E|^2_F|U^\top|_{op}^2=|E|^2_F|U|_{op}^2=(\sum_t{\|e_t\|^2_2})|U|_{op}^2 \leq T\delta^2 |U|_{op}^2,
\label{eq:final-bound}
\end{equation}
which gives \eqref{eq:singular-bound} dividing by $T$.
For \eqref{eq:prob-dist}, set $\mathcal N:=\{\operatorname{softmax}(Uw):w\in\mathcal S\}$, the image of
$\mathcal S$ under the smooth map $w\mapsto\operatorname{softmax}(Uw)$; since $\dim\mathcal S\le2$, this set is smoothly parametrised by at most two parameters. Because $\tilde h_t\in\mathcal S$, the point $\operatorname{softmax}(\widetilde L_t)$ lies in $\mathcal N$ and is admissible in the distance.  Row-wise, $L_t-\widetilde L_t=U(h_t-\tilde h_t)=Ue_t$, so $\|L_t-\widetilde L_t\|_2\le\|U\|_{\mathrm{op}}\delta$, and Lemma~\ref{lem:lip} gives $\operatorname{dist}_{\ell_2}(p_\theta(\cdot\mid x_{<t}),\mathcal N) \le\tfrac12\|U\|_{\mathrm{op}}\delta$. Squaring and averaging over $t$ gives \eqref{eq:prob-dist}.
\end{proof}

\subsection{Proof of Theorem~\ref{thm:transversal}}
\label{app:transversal}

\Tranversal*

\begin{proof}
 
Since $e_0 = 0$ by construction ($h_{t_{\mathrm{start}}}$ is the starting anchor and $d_0 = h_{t_{\mathrm{start}}} - h_{t_{\mathrm{start}}} = 0$), we have $V_0 = \|e_0\|_2^2 / \|v_{\mathrm{geo}}\|_2^2 = 0$. The exact recurrence holds at every step:
\begin{equation}
  V_{t+1} \;=\; \gamma V_t + \tau + \varepsilon_t,
  \qquad t = 0, 1, \ldots, T-1,
  \label{eq:exact-rec-app}
\end{equation}
where $\varepsilon_t$ is defined in \eqref{eq:innovation}.

\textbf{Softplus lower bound.} For every $x\in\mathbb{R}$,
\begin{equation}
  \log(1+e^{x}) \;\ge\; x^{+} := \max(x,0),
  \label{eq:sp-lb}
\end{equation}
since $\log(1+e^{x}) \ge \log e^{x} = x$ and $\log(1+e^{x}) \ge \log 1 = 0$.

\textbf{Proof of part (i): Mean confinement.}

Sum \eqref{eq:exact-rec-app} over $t = 0, \ldots, T-1$:
\begin{equation}
  \sum_{t=0}^{T-1} V_{t+1} \;=\; \gamma \sum_{t=0}^{T-1} V_t + T\tau + \sum_{t=0}^{T-1}\varepsilon_t.
  \label{eq:sum-rec-app}
\end{equation}
Write $\sum_{t=0}^{T-1} V_{t+1} = \sum_{t=1}^{T} V_t = \sum_{t=0}^{T-1} V_t - V_0 + V_T = \sum_{t=0}^{T-1} V_t + V_T$ (using $V_0 = 0$). Substituting into \eqref{eq:sum-rec-app}:
\begin{equation}
  \sum_{t=0}^{T-1} V_t + V_T \;=\; \gamma \sum_{t=0}^{T-1} V_t + T\tau + \sum_{t=0}^{T-1}\varepsilon_t.
\end{equation}
Rearranging and using $V_T \ge 0$:
\begin{equation}
  (1-\gamma)\sum_{t=0}^{T-1} V_t \;=\; -V_T + T\tau + \sum_{t=0}^{T-1}\varepsilon_t \;\le\;
  T\tau + \sum_{t=0}^{T-1}\varepsilon_t.
  \label{eq:rearranged-app}
\end{equation}

We construct the bound for the violation sum. Since $\varepsilon_t \le \varepsilon_t^+ := \max(\varepsilon_t, 0)$ and the softplus lower bound \eqref{eq:sp-lb} gives $\log(1 + e^{\varepsilon_t}) \ge \varepsilon_t^+$:
\begin{equation}
  \sum_{t=0}^{T-1}\varepsilon_t\;\le\;\sum_{t=0}^{T-1}\varepsilon_t^+\;\le\;\sum_{t=0}^{T-1 }\log(1+e^{\varepsilon_t}) \;=\; T\,\mathcal{L}_{\mathrm{Lyap}}(\theta).
  \label{eq:violation-bound-app}
\end{equation}

Substituting \eqref{eq:violation-bound-app} into
\eqref{eq:rearranged-app} and dividing by $T(1-\gamma)$:
\begin{equation}
  \bar{V}_T \;=\; \frac{1}{T}\sum_{t=0}^{T-1}V_t \;\le\; \frac{\tau}{1-\gamma} + \frac{\mathcal{L}_{\mathrm{Lyap}}(\theta)}{1-\gamma}.
\end{equation}
Note that the $V_0/T$ transient term present in the general case vanishes here because $V_0 = 0$.
\hfill$\square_{(i)}$

\textbf{Proof of part (ii): Variance control.}

Throughout, the trajectory is a single contiguous segment $\mathcal T=\{0,\ldots,T-1\}$ with $V_t\ge 0$, and we write $x^{+}:=\max(x,0)$, $x^{-}:=\max(-x,0)$, so that $\varepsilon_t=\varepsilon_t^{+}-\varepsilon_t^{-}$ and $\varepsilon_t^2=(\varepsilon_t^{+})^2+(\varepsilon_t^{-})^2$. Recall from Section~\ref{sec:theorem} that $M=\max_t\varepsilon_t^{+}$ is the largest ISS violation and $n=\frac{M^2}{\log(1+e^{M})}$. Let $\bar\varepsilon:=\frac1T\sum_{t=0}^{T-1}\varepsilon_t$ and
\begin{equation}
  P:=\frac1T\sum_{t=0}^{T-1}(\varepsilon_t^{+})^2,
  \qquad
  D:=\frac1T\sum_{t=0}^{T-1}(\varepsilon_t^{-})^2 .
  \label{eq:PD-app}
\end{equation}

Dividing the equality in \eqref{eq:rearranged-app} by $T$ gives 
\begin{equation}
  (1-\gamma)\bar V_T=\tau+\bar\varepsilon-\frac{V_T}{T}.
  \label{eq:mean-identity-app}
\end{equation}
Let $\widetilde V_t:=V_t-\bar V_T$ and $\delta_t:=\varepsilon_t-\bar\varepsilon$.
Subtracting $\bar V_T$ from \eqref{eq:exact-rec-app} and using
\eqref{eq:mean-identity-app} yields $\widetilde V_{t+1}=\gamma\widetilde V_t+\delta_t+V_T/T$. Unrolling from $\widetilde V_0=-\bar V_T$,
\begin{equation}
  \widetilde V_t=S_t+R_t,\qquad
  S_t:=\sum_{s=0}^{t-1}\gamma^{t-1-s}\delta_s,\qquad
  R_t:=-\gamma^{t}\bar V_T+\frac{V_T(1-\gamma^{t})}{T(1-\gamma)}.
  \label{eq:unroll-app}
\end{equation}
By $(a+b)^2\le 2a^2+2b^2$,
\begin{equation}
  \operatorname{Var}_t[V_t]=\frac1T\sum_{t=0}^{T-1}\widetilde V_t^{\,2}
  \le \frac2T\sum_{t=0}^{T-1}S_t^2+\mathcal R_T,
  \qquad
  \mathcal R_T:=\frac2T\sum_{t=0}^{T-1}R_t^2 .
  \label{eq:var-split-app}
\end{equation}

The sequence $S_t$ is the convolution of $(\gamma^k)_{k\ge0}$, whose $\ell^1$-norm is $1/(1-\gamma)$, with $(\delta_s)_{s=0}^{T-1}$. Young's inequality $\|a*b\|_{\ell^2}\le\|a\|_{\ell^1}\|b\|_{\ell^2}$ gives
\begin{equation}
  \frac1T\sum_{t=0}^{T-1}S_t^2
  \le\frac{1}{(1-\gamma)^2}\,\frac1T\sum_{t=0}^{T-1}\delta_t^2
  =\frac{1}{(1-\gamma)^2}\Bigl(\frac1T\sum_{t=0}^{T-1}\varepsilon_t^2-\bar\varepsilon^{2}\Bigr)
  \le\frac{P+D}{(1-\gamma)^2}.
  \label{eq:young-app}
\end{equation}

To bound $P$, let $q(x):=\frac{x^2}{\log(1+e^{x})}$ for $x\ge0$, so that $n=q(M)$. With $\sigma(x):=\frac{1}{1+e^{-x}}$, the derivative of $\log(1+e^{x})$,
\[
  q'(x)=\frac{x\bigl(2\log(1+e^{x})-x\,\sigma(x)\bigr)}{\log^2(1+e^{x})}\ge 0,
\]
since $\log(1+e^{x})\ge x\ge x\,\sigma(x)$ by \eqref{eq:sp-lb}. Hence $q$ is nondecreasing on $[0,\infty)$. Whenever $\varepsilon_t>0$, we have $0<\varepsilon_t\le M$ and therefore $(\varepsilon_t^{+})^2=q(\varepsilon_t)\log(1+e^{\varepsilon_t})
\le q(M)\log(1+e^{\varepsilon_t})=n\log(1+e^{\varepsilon_t})$;
the same inequality is trivial when $\varepsilon_t\le0$. Averaging,
\begin{equation}
  P\le n\,\mathcal L_{\mathrm{Lyap}}(\theta).
  \label{eq:P-bound-app}
\end{equation}

To bound $D$, note that $V_{t+1}\ge0$ implies $\varepsilon_t^{-}\le\gamma V_t+\tau$, so $D\le 2\gamma^2\,\frac1T\sum_{t=0}^{T-1} V_t^2+2\tau^2$. From $V_{t+1}\le\gamma V_t+\tau+\varepsilon_t^{+}$ and $V_0=0$,
\[
  0\le V_t\le\frac{\tau}{1-\gamma}+\sum_{s=0}^{t-1}\gamma^{t-1-s}\varepsilon_s^{+}.
\]
Squaring and applying Young's inequality as in \eqref{eq:young-app} gives
$\frac1T\sum_{t=0}^{T-1} V_t^2\le\frac{2(\tau^2+P)}{(1-\gamma)^2}$. Therefore
\begin{equation}
  D\le 2\tau^2+\frac{4\gamma^2(\tau^2+P)}{(1-\gamma)^2}.
  \label{eq:D-bound-app}
\end{equation}

Combining \eqref{eq:var-split-app}, \eqref{eq:young-app} and \eqref{eq:D-bound-app},
\[
  \operatorname{Var}_t[V_t]
  \le\frac{2P}{(1-\gamma)^2}\Bigl[1+\frac{4\gamma^2}{(1-\gamma)^2}\Bigr]
  +\frac{4\tau^2}{(1-\gamma)^2}\Bigl[1+\frac{2\gamma^2}{(1-\gamma)^2}\Bigr]
  +\mathcal R_T ,
\]
and \eqref{eq:P-bound-app} yields \eqref{eq:var-bound}.

Finally, from \eqref{eq:unroll-app}, using $\sum_{t\ge0}\gamma^{2t}=1/(1-\gamma^2)$ and $(1-\gamma^t)^2\le1$,
\begin{equation}
  \mathcal R_T\le\frac{4\bar V_T^{2}}{T(1-\gamma^2)}+\frac{4V_T^{2}}{T^2(1-\gamma)^2}.
  \label{eq:RT-bound-app}
\end{equation}
If $M\le M_0$ for all $T$, then unrolling $V_{t+1}\le\gamma V_t+\tau+M_0$ gives $V_t\le(\tau+M_0)/(1-\gamma)$ for all $t$. Hence $\bar V_T$ and $V_T$ are bounded uniformly in $T$, and \eqref{eq:RT-bound-app} gives $\mathcal R_T=O(1/T)$.
\hfill$\square_{(ii)}$

\textbf{Proof of part (iii): Practical tube width.}

By Jensen's inequality applied to the concave function $\sqrt{\cdot}$:
\begin{equation}
  \mathbb{E}_t\!\left[\|e_t\|_2\right]
  \;=\;
  \mathbb{E}_t\!\left[\sqrt{V_t}\,\|v_{\mathrm{geo}}\|_2\right]
  \;\le\;
  \sqrt{\mathbb{E}_t[V_t]}\cdot\|v_{\mathrm{geo}}\|_2
  \;=\;
  \sqrt{\bar{V}_T}\cdot\|v_{\mathrm{geo}}\|_2.
\end{equation}
Substituting \eqref{eq:mean-bound} (with $V_0 = 0$):
\begin{equation}
  \mathbb{E}_t\!\left[\|e_t\|_2\right]
  \;\le\;
  \|v_{\mathrm{geo}}\|_2\cdot
  \sqrt{\frac{\tau + \mathcal{L}_{\mathrm{Lyap}}(\theta)}{1-\gamma}}.
\end{equation}
\hfill$\square_{(iii)}$

\end{proof}


\section{Additional Experiment Results}
\label{app:exp-details}
\subsection{Implicit Tube Preservation at Inference Time}
To further investigate the intrinsic geometric structure learned by the models, we visualize the trajectory dynamics during autoregressive generation (Figure \ref{fig:e_t_vs_v_t}). Crucially, during inference, models operate in an open-loop regime where all regularization losses are entirely deactivated. The hidden states are propelled solely by the vector field learned by the network weights.

\textbf{Baseline Drift:} The standard fine-tuned model (Regular, orange) exhibits substantial transversal drift across all samples, confirming the susceptibility of unconstrained models to inference-time Brownian motion and mode collapse.

\textbf{STP Volatility:} While Semantic Tube Prediction (STP, blue) generally reduces the mean error compared to the baseline, its generation trajectories are highly erratic. This empirical volatility is a direct consequence of its heuristic formulation: by solely penalizing local angular deviations without imposing global magnitude bounds, STP fails to construct a monotone structural boundary, allowing the transversal error to oscillate wildly.

\textbf{ControlJEPA Stability:} In stark contrast, ControlJEPA (green) consistently traces the smoothest trajectories with the lowest mean transversal errors. Despite the absence of an active Lyapunov penalty during decoding, the hidden states remain tightly confined within the semantic tube. This provides compelling evidence that our Input-to-State Stability (ISS) optimization does not merely memorize training paths; rather, it fundamentally reshapes the representation manifold by embedding a global attractor that implicitly corrects transversal drift step-by-step at inference time.
\begin{figure}[t]
    \centering
    \includegraphics[width=0.9\columnwidth]{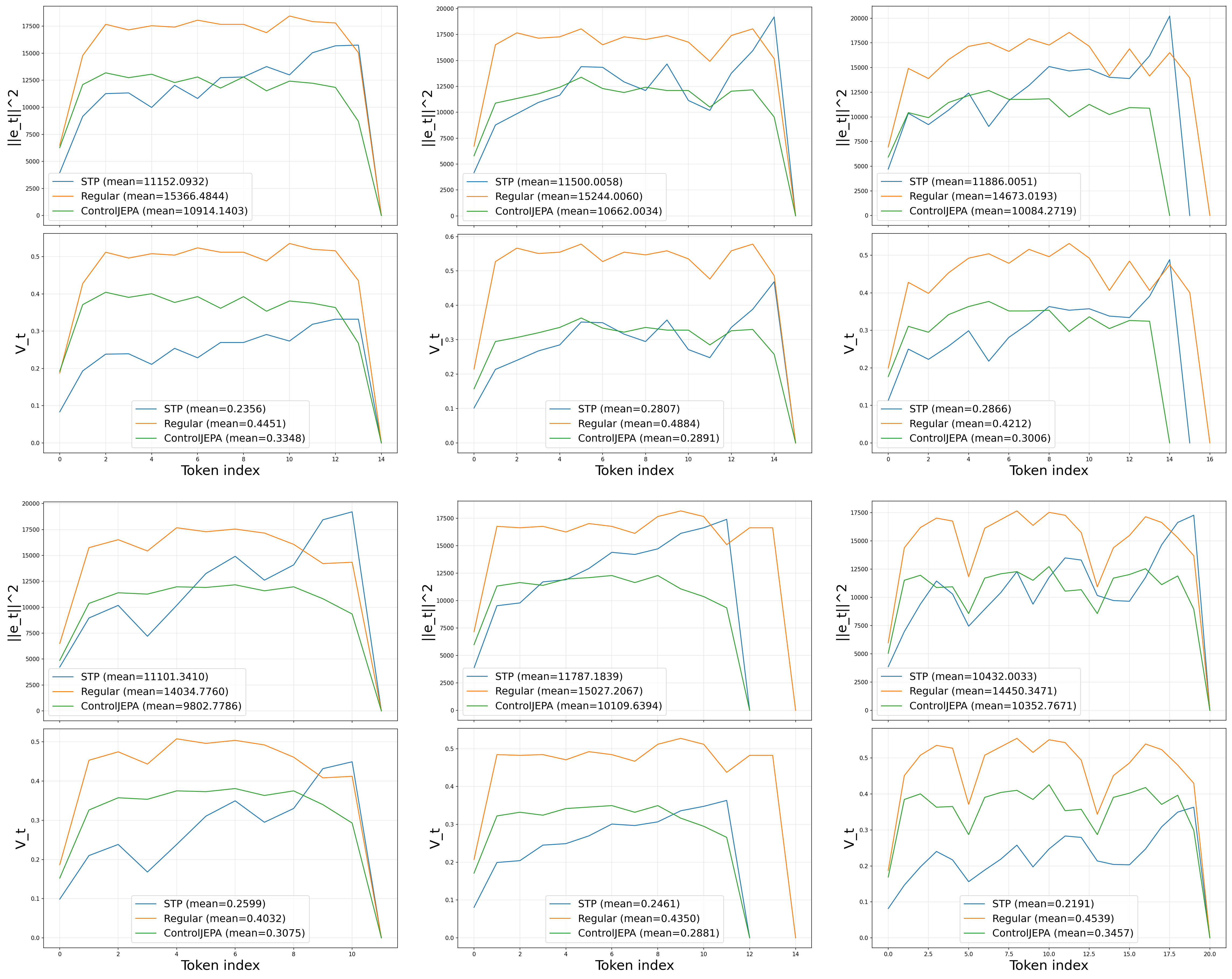}
    \caption{Step-by-step evolution of the unnormalized transversal error $\|e_t\|_2^2$ (top panels) and the dimensionless Lyapunov candidate $V_t$ (bottom panels) across multiple generated sequences. Operating in an open-loop regime without any active loss constraints, ControlJEPA (green) uniquely maintains a tight, stable semantic tube. In contrast, regular fine-tuning (orange) exhibits severe geometric drift, and STP (blue) suffers from high trajectory volatility.}
    \label{fig:e_t_vs_v_t}
    \vspace{-2em}
\end{figure}
\subsection{Hyperparameter Sensitivity: Contraction Rate and Orthogonal Slack}
To comprehensively evaluate the robustness and behavior of the decoupled Lyapunov hyperparameters, we perform a grid search over the contraction rate $\gamma$ and the orthogonal slack $\tau$. As illustrated in Figure \ref{fig:gamma_tau}, ControlJEPA exhibits a wide, flat Pareto-optimal stability region across most combinations, freeing practitioners from brittle hyperparameter tuning.
Crucially, an analysis of the peak performance regions reveals a profound connection between the optimal tube geometry and the intrinsic semantic diversity of the downstream task:

\textbf{Rigid Geodesics for Rule-Based Tasks:} For tasks governed by strict syntactic rules or deterministic logic with low semantic variance (e.g., SYNTH), accuracy peaks at a tighter configuration ($\gamma \in [0.84, 0.9]$ combined with smaller $\tau$). Because the ground-truth geodesic for these tasks is highly rigid, a faster contraction rate aggressively squashes transversal noise and forces strict adherence to the target trajectory, yielding maximal accuracy.

\textbf{Relaxed Envelopes for Semantic Diversity:} Conversely, for open-domain tasks characterized by high semantic diversity, complex reasoning, and linguistic nuance (e.g., NQ-Open, HellaSwag), the optimal configuration systematically shifts toward a more relaxed boundary ($\gamma \in [0.9, 0.95]$ coupled with larger $\tau$). High-diversity tasks naturally possess meandering semantic trajectories. Enforcing an overly aggressive contraction would over-regularize the representations, leading to the dimensionality collapse discussed in Section 6.1. A larger $\gamma$ (slower contraction) and a larger $\tau$ (greater slack) provide the necessary geometric volume to accommodate these rich manifold variations while still preventing catastrophic trajectory drift.

\begin{figure}[t]
    \centering
    \includegraphics[width=0.85\columnwidth]{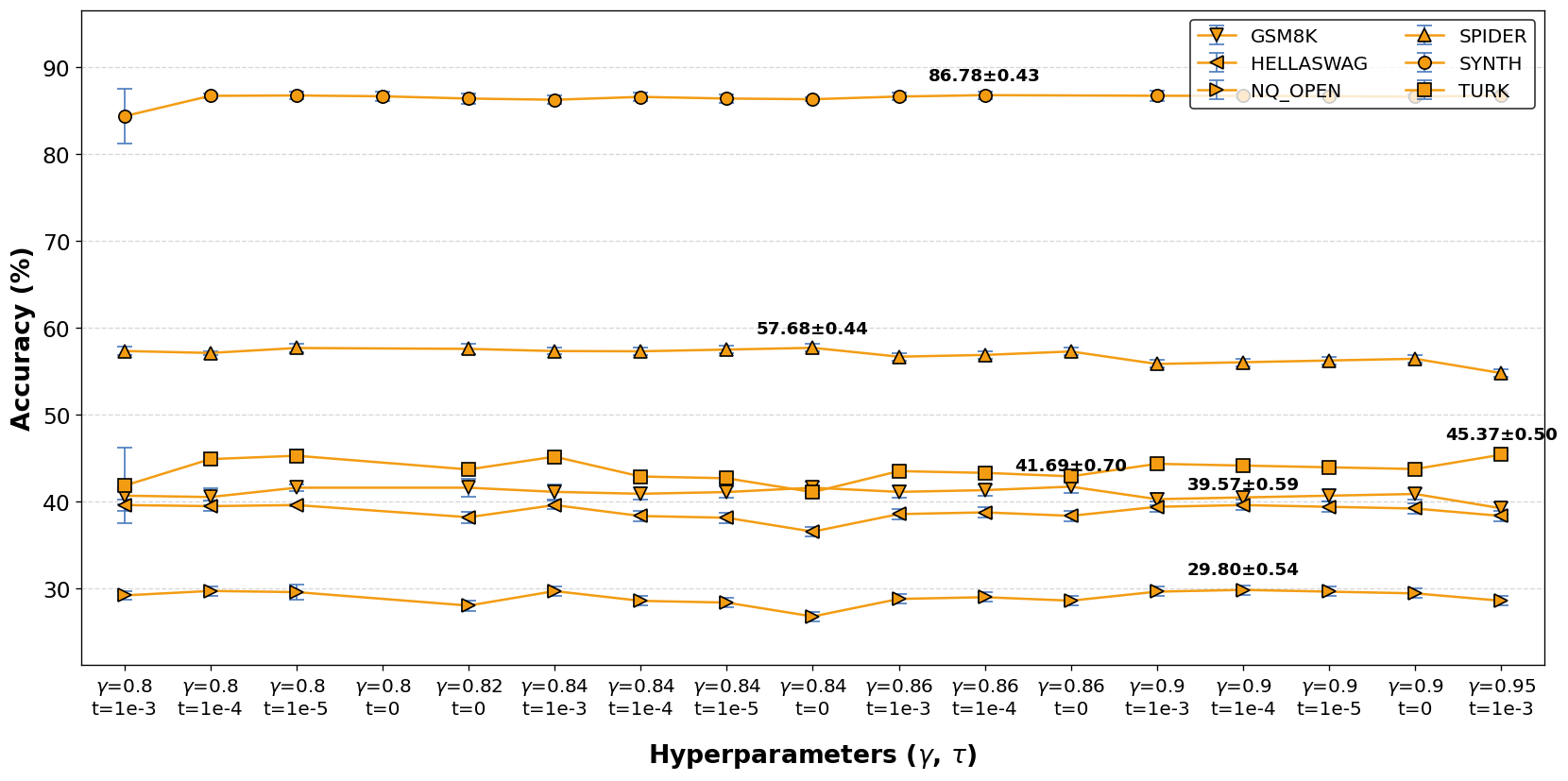}
    \caption{The accuracy landscape across diverse datasets demonstrates a broad stability region. Tasks with rigid, rule-based structures (e.g., SYNTH) peak at tighter tube configurations ($\gamma \in [0.84, 0.9]$), while semantically diverse open-domain tasks (e.g., NQ-Open) require a more relaxed envelope ($\gamma \ge 0.9$ with larger $\tau$) to accommodate complex manifold variations.}
    \label{fig:gamma_tau}
    \vspace{-1em}
\end{figure}

\subsection{Hardware Settings}
\label{hardware}
We conducted all experiments on two NVIDIA H100 80GB GPUs.
\section{Broader Societal Impacts}
\label{app:broader-impact}
This work is primarily methodological and theoretical: it introduces ControlJEPA, a trajectory regularization framework based on discrete-time Input-to-State Stability for improving the geometry, stability, and data efficiency of hidden-state trajectories in language models. Its societal impact is therefore indirect. On the positive side, by providing explicit and interpretable control over trajectory deviation, ControlJEPA may help researchers design more reliable and sample-efficient models, reduce trial-and-error fine-tuning, and better diagnose representation-level failure modes such as trajectory drift, collapse, or instability. However, improved language-model performance and data efficiency may also amplify broader risks associated with capable generative systems, including misinformation, misuse, overreliance, and deployment in high-stakes settings without adequate safeguards. Our method does not directly address alignment, fairness, privacy, factuality, robustness to adversarial inputs, or harmful-content prevention; the theoretical guarantees apply to the controlled hidden-state geometry rather than to safe or truthful model behavior. Thus, ControlJEPA should be viewed as a foundation for improving representation stability, not as a standalone safety guarantee for deployed language models.


\newpage
\newpage
\section*{NeurIPS Paper Checklist}

\begin{enumerate}

\item {\bf Claims}
    \item[] Question: Do the main claims made in the abstract and introduction accurately reflect the paper's contributions and scope?
    \item[] Answer: \answerYes{} 
    \item[] Justification: The abstract and introduction clearly state the paper’s three main contributions: the ISS-based trajectory regularization framework, the Lyapunov Tube Loss, and the theoretical/empirical analysis of stability and representational capacity. These claims are supported by theoretical derivations in Sections~\ref{sec:method}--\ref{sec:theoretical_analysis} and empirical validation in Sections~\ref{sec:experiments}--~\ref{sec:analysis}.

\item {\bf Limitations}
    \item[] Question: Does the paper discuss the limitations of the work performed by the authors?
    \item[] Answer: \answerYes{} 
    \item[] Justification: The paper explicitly discusses limitations in the “Limitations” paragraph of Section~\ref{sec:conclusion}, including lack of guarantees for out-of-distribution trajectories, inference-time stability remaining empirical, scalability only tested up to 8B models, and dependence on dataset-provided semantic anchors.

\item {\bf Theory assumptions and proofs}
    \item[] Question: For each theoretical result, does the paper provide the full set of assumptions and a complete (and correct) proof?
    \item[] Answer: \answerYes{} 
    \item[] Justification: All theoretical statements are formally presented with numbered definitions, propositions, and theorems. Assumptions are explicitly stated (e.g., Assumption 1 in Appendix A.2), and complete proofs are provided in Appendix~\ref{app:transversal}. 

    \item {\bf Experimental result reproducibility}
    \item[] Question: Does the paper fully disclose all the information needed to reproduce the main experimental results of the paper to the extent that it affects the main claims and/or conclusions of the paper (regardless of whether the code and data are provided or not)?
    \item[] Answer: \answerYes{} 
    \item[] Justification: The paper specifies datasets, model families, baseline methods, training objectives, and hyperparameters (\(\lambda, \gamma, \tau\)) in Section~\ref{setup}. Additional implementation and proof details are included in the appendix and supplementary material.

\item {\bf Open access to data and code}
    \item[] Question: Does the paper provide open access to the data and code, with sufficient instructions to faithfully reproduce the main experimental results, as described in supplemental material?
    \item[] Answer: \answerYes{} 
    \item[] Justification: We provide anonymized code, datasets, training scripts, and reproduction instructions in the Supplementary Material. All experimental assets required to reproduce the main results are included in the released repository.

\item {\bf Experimental setting/details}
    \item[] Question: Does the paper specify all the training and test details (e.g., data splits, hyperparameters, how they were chosen, type of optimizer) necessary to understand the results?
    \item[] Answer: \answerYes{} 
    \item[] Justification: The paper details datasets, model scales, baselines, and tuned hyperparameters in Section~\ref{setup}. It also explains the optimization behavior and hyperparameter sensitivity in Section~\ref{hyper}.


\item {\bf Experiment statistical significance}
    \item[] Question: Does the paper report error bars suitably and correctly defined or other appropriate information about the statistical significance of the experiments?
    \item[] Answer: \answerYes{} 
    \item[] Justification: The paper reports error bars for the main experimental results in Figures~\ref{fig:acc}--~\ref{fig:infer_time}. The reported error bars reflect variability across experimental runs and are described alongside the corresponding figures in Sections~\ref{sec:experiments} and ~\ref{sec:analysis}

\item {\bf Experiments compute resources}
    \item[] Question: For each experiment, does the paper provide sufficient information on the computer resources (type of compute workers, memory, time of execution) needed to reproduce the experiments?
    \item[] Answer: \answerYes{} 
    \item[] Justification: We have provided information about computing resources needed in Appendix~\ref{hardware}
    
\item {\bf Code of ethics}
    \item[] Question: Does the research conducted in the paper conform, in every respect, with the NeurIPS Code of Ethics \url{https://neurips.cc/public/EthicsGuidelines}?
    \item[] Answer: \answerYes{} 
    \item[] Justification: The research involves theoretical analysis and fine-tuning of publicly available language models and datasets, with no human subject experiments or unsafe deployment claims. The work appears consistent with the NeurIPS Code of Ethics.

\item {\bf Broader impacts}
    \item[] Question: Does the paper discuss both potential positive societal impacts and negative societal impacts of the work performed?
    \item[] Answer: \answerYes{} 
    \item[] Justification: The paper discusses both positive and negative societal impacts in Appendix~\ref{app:broader-impact}. It highlights improved reliability and data efficiency for language models, while acknowledging broader risks such as misuse, misinformation, and unsafe deployment of capable generative systems.
    
\item {\bf Safeguards}
    \item[] Question: Does the paper describe safeguards that have been put in place for responsible release of data or models that have a high risk for misuse (e.g., pre-trained language models, image generators, or scraped datasets)?
    \item[] Answer: \answerNA{} 
    \item[] Justification: The paper proposes a training-time regularization loss and does not release new models, datasets, or tools that pose risks for misuse.

\item {\bf Licenses for existing assets}
    \item[] Question: Are the creators or original owners of assets (e.g., code, data, models), used in the paper, properly credited and are the license and terms of use explicitly mentioned and properly respected?
    \item[] Answer: \answerYes{} 
    \item[] Justification: All models, benchmarks, and baseline methods are properly cited with their original references throughout the paper and in the bibliography.

\item {\bf New assets}
    \item[] Question: Are new assets introduced in the paper well documented and is the documentation provided alongside the assets?
    \item[] Answer: \answerYes{} 
    \item[] Justification: The paper releases an implementation of ControlJEPA together with documentation, configuration files, and reproduction instructions in the supplementary material.

\item {\bf Crowdsourcing and research with human subjects}
    \item[] Question: For crowdsourcing experiments and research with human subjects, does the paper include the full text of instructions given to participants and screenshots, if applicable, as well as details about compensation (if any)? 
    \item[] Answer: \answerNA{} 
    \item[] Justification: The paper does not involve crowdsourcing nor research with human subjects.

\item {\bf Institutional review board (IRB) approvals or equivalent for research with human subjects}
    \item[] Question: Does the paper describe potential risks incurred by study participants, whether such risks were disclosed to the subjects, and whether Institutional Review Board (IRB) approvals (or an equivalent approval/review based on the requirements of your country or institution) were obtained?
    \item[] Answer: \answerNA{} 
    \item[] Justification: The paper does not involve crowdsourcing or research with human subjects, so IRB approval is not applicable.

\item {\bf Declaration of LLM usage}
    \item[] Question: Does the paper describe the usage of LLMs if it is an important, original, or non-standard component of the core methods in this research? Note that if the LLM is used only for writing, editing, or formatting purposes and does \emph{not} impact the core methodology, scientific rigor, or originality of the research, a declaration is not required.
    \item[] Answer: \answerNA{} 
    \item[] Justification: LLMs were used only as a writing assistant for editing and polishing the manuscript (grammar, clarity, concision).

\end{enumerate}

\end{document}